\documentclass{article}

\usepackage{log_2023}             

\usepackage{booktabs}            
\usepackage{multirow}            
\usepackage{amsfonts} 
\usepackage{graphicx,csquotes}   
\usepackage{duckuments}         
\usepackage[numbers,compress,sort]{natbib}
\usepackage{amsmath,mathrsfs}
\usepackage{amssymb}
\usepackage{mathtools, pifont}
\usepackage[backref=page]{hyperref}
\usepackage{amsthm}
\usepackage[capitalize,noabbrev]{cleveref}
\newcommand{\wh}{\widehat}
\newcommand{\wt}{\widetilde}
\newcommand{\G}{\mathcal{G}}

\newcommand{\I}{\mathcal{I}}
\newcommand{\X}{\mathcal{X}}
\newcommand{\V}{\mathcal{V}}
\newcommand{\E}{\mathcal{E}}
\newcommand{\W}{\mathcal{W}}
\newcommand{\M}{\mathbf{M}}
\newcommand{\R}{\mathbb{R}}
\newcommand{\e}{\epsilon}

\newcommand{\PD}{\rm{PD}}

\newtheorem{theorem}{Theorem}

\newtheorem*{theorem*}{Theorem}

\title{EMP: Effective Multidimensional Persistence for Graph Representation Learning}

\author[I.S. Dominguez et al.]{%
Ignacio Segovia-Dominguez\thanks{Equal contribution.}\\
{West Virginia University}\\
\email{Ignacio.SegoviaDominguez@mail.wvu.edu}\And
Yuzhou Chen\footnotemark[1]\\
{Temple University}\\
\email{yuzhou.chen@temple.edu}\And
Cuneyt G. Akcora\\
{University of Central Florida}\\
\email{cuneyt.akcora@ucf.edu}\And
Zhiwei Zhen\\
{University of Texas at Dallas}\\
\email{zhiwei.zhen@utdallas.edu}\And
Murat Kantarcioglu\\
{University of Texas at Dallas}\\
\email{muratk@utdallas.edu}\And
Yulia R. Gel\\
{University of Texas at Dallas}, USA\\
\email{ygl@utdallas.edu}\And
Baris Coskunuzer\\
{University of Texas at Dallas}\\
\email{coskunuz@utdallas.edu}
}

\begin{document}
\maketitle

\begin{abstract}

Topological data analysis (TDA) is gaining prominence across a wide spectrum of machine learning  tasks that spans from manifold learning to graph classification. A pivotal technique within TDA is persistent homology (PH), which furnishes an exclusive topological imprint of data by tracing the evolution of latent structures as a scale parameter changes. Present PH tools are confined to analyzing data through a single filter parameter. However, many scenarios necessitate the consideration of multiple relevant parameters to attain finer insights into the data. We address this issue by introducing the Effective Multidimensional Persistence (EMP) framework. This framework empowers the exploration of data by simultaneously varying multiple scale parameters. The framework integrates descriptor functions into the analysis process, yielding a highly expressive data summary. It seamlessly integrates established single PH summaries into multidimensional counterparts like EMP Landscapes, Silhouettes, Images, and Surfaces. These summaries represent data's multidimensional aspects as matrices and arrays, aligning effectively with diverse ML models. We provide theoretical guarantees and stability proofs for EMP summaries. We demonstrate EMP's utility in graph classification tasks, showing its effectiveness. Results reveal that EMP enhances various single PH descriptors, outperforming cutting-edge methods on multiple benchmark datasets.

\end{abstract}

\section{Introduction}\label{sec:intro}

In the past decade, topological data analysis proved to be a powerful machinery to discover hidden patterns in various forms of data that are otherwise inaccessible with more traditional methods~\cite{chen2022time,akcora2021blockchain}.  In particular, for graph machine learning (ML) tasks, while many traditional methods fail, TDA and, specifically, tools of persistent homology (PH), have demonstrated a high potential to detect local and global patterns and to produce a unique topological fingerprint to be used in various ML tasks~\cite{chen2021tamp}. 
This makes PH particularly attractive for capturing various characteristics of the complex data which may play the key role behind the learning task performance.

In turn, multiparameter persistence (multipersistence or MP) is a novel idea to further advance the PH machinery by analyzing the data in a much finer way simultaneously along multiple dimensions. However, due to technical issues stemming from its multidimensional structure in commutative algebra, a general definition for MP has not yet been established. In this paper, we develop an alternative approach to utilize multipersistence ideas efficiently for various types of data, with a main focus on graph representation. In particular, we bypass technical issues with the MP (\cref{sec:MP_theory}) by extracting very practical summaries via the slicing approach in a structured way. We then obtain computationally efficient multidimensional topological fingerprints of the data in the form of matrices and linear arrays which can serve as input to ML models.

Our main contribution lies in employing interpretable persistence features to achieve accurate graph classification. This is crucial because existing graph neural networks (GNNs) demand extensive data and costly training, yet lack insights into specific classification outcomes. Unlike GNNs, our single and multi-parameter persistence features offer insights into the evolving underlying graphs as the scale parameter changes, enabling the creation of efficient machine learning models. In domains with limited graph data, the predictive models we construct using topological summaries from single and multi-parameter persistence achieve accuracy levels comparable to the top GNN models.

\smallskip
{\bf Our contributions} can be summarized as follows:

\begin{itemize}
 
    \item[\ding{227}] We develop a computationally efficient and highly expressive EMP framework that provides multidimensional topological fingerprints of the data. EMP expands popular summaries of single persistence to multidimensions by adapting an effective slicing direction.
     
     \item[\ding{227}] We derive theoretical stability guarantees of the new topological summaries.
     
    \item[\ding{227}] We illustrate the utility of EMP summaries in various settings and compare our results to state-of-the-art (SOTA) methods. Our numerical experiments demonstrate that EMP summaries outperform SOTA in a broad range of benchmark datasets for graph classification tasks.

\end{itemize}

\section{Related Work} 
\label{sec:relatedwork}

\subsection{Persistent Homology} \label{sec:MP_related}

Persistent homology is a key tool in topological data analysis to deliver invaluable and complementary information on the intrinsic properties of data that are inaccessible with conventional methods~\citep{chazal2021introduction,hensel2021survey}. In the past decade, PH has become quite popular in various ML tasks, ranging from manifold learning to medical image analysis, and material science to finance (TDA applications library~\citep{giunti22}).

Multipersistence is a highly promising approach with the potential to significantly improve the success of single parameter persistence (SP) in applications of TDA, but there exist some fundamental challenges related to converting this novel idea into an effective feature extraction method, as follows. Except for some special cases, MP theory tends to suffer from the problem of the nonexistence of barcode decomposition because of the partially ordered structure of the index set $\{(\alpha_i,\beta_j)\}$~\citep{botnan2022introduction}. Lesnick et al. \cite{lesnick2015interactive} proposed a solution to circumvent this problem using the \textit{slicing technique}. This method involves examining one-dimensional fibers within the multiparameter domain. In these fibers, the multidimensional persistence module is constrained to a single direction (referred to as a \textit{slice}), and single persistence analysis is applied to this one-dimensional slice. Later, by using this novel idea, Carri{\`e}re et al.~\cite{carriere2020multiparameter} combined several slicing directions (vineyards) and obtained a vectorization by summarizing the persistence diagrams (PDs) in these directions. There are several promising recent studies in this direction~\citep{botnan2021signed, vipond2020multiparameter}. However, these methods often come with several drawbacks, primarily related to computational expenses. Consequently, their practical use for real-world problems is constrained. Very recently, a couple of very successful MP vectorizations proved the potential of MP in several settings~\cite{loiseaux2023framework, loiseaux2023stable}. Here, we aim to add a practical and highly efficient way to use MP approach for various forms of data and provide a multidimensional topological vectorization with EMP summaries.

\subsection{Graph Neural Networks} \label{sec:graph_related}

After the success of convolutional neural networks (CNN) on image-based tasks, graph neural networks (GNNs) have emerged as powerful tools for graph-level classification and representation learning. Based on the spectral graph theory, Bruna et al. introduced a graph-based convolution in the Fourier domain~\cite{bruna2013spectral} . However, the complexity of this model is high since all Laplacian eigenvectors are needed. To tackle this problem, ChebNet~\cite{defferrard2016convolutional} integrated spectral graph convolution with Chebyshev polynomials. Then, Graph Convolutional Networks (GCNs) of~\cite{kipf2016semi} simplified the graph convolution with a localized first-order approximation. More recently, various approaches proposed accumulating graph information from a wider neighborhood, using diffusion aggregation and random walks. Such higher-order methods include approximate personalized propagation of neural predictions (APPNP)~\cite{klicpera2018predict}, higher-order graph convolutional architectures (MixHop)~\cite{abu2019mixhop}, multi-scale graph convolution (N-GCN)~\cite{abu2020n}, and 
 L\'evy Flights Graph Convolutional Networks (LFGCN)~\cite{LFGCN}.
In addition to random walks, other recent approaches include GNNs on directed graphs (MotifNet)~\cite{monti2018motifnet}, graph convolutional networks with attention mechanism (GAT, SPAGAN)~\cite{velivckovic2017graph,yang2019spagan}, and graph Markov neural network (GMNN)~\cite{qu2019gmnn}.  While GNNs produce great performances in many graph learning tasks, they tend to suffer from over-smoothing problems and are vulnerable to graph perturbations.

\section{Background} \label{sec:background}

We start with providing the basic background for our framework.  Since our primary focus pertains to graph representation learning in this study, we elucidate our methodology within the context of graphs. For a comprehensive treatment of the classical persistent homology construction in various contexts, consult~\cite{dey2022computational}. Due to space limitations, our methodology on alternative data forms, such as point clouds and image data are explained in \Cref{sec:othertypedata,sec:otherfiltration}.

\subsection{Persistent Homology} \label{sec:PH}

Persistent Homology  is a mathematical machinery to capture the hidden shape patterns in the data by using algebraic topology tools. PH extracts this information by keeping track of evolution of the topological features (components, loops, cavities) created in the data while looking at it in different resolutions~\cite{chazal2021introduction}.

For a given graph $\G$, consider a nested sequence of subgraphs $\G_1 \subseteq \ldots \subseteq \G_N=\G$. For each $\G_i$, define an abstract simplicial complex 
$\wh{\G}_{i}$, $1\leq i\leq N$, yielding a filtration of complexes $\wh{\G}_{1} \subseteq \ldots \subseteq \wh{\G}_{N}$. 
Here, clique complexes are among the most common ones, i.e., a clique complex $\wh{\G}$ is obtained by assigning (filling with) a $k$-simplex to each complete $(k+1)$-complete subgraph in $\G$. For example,  a $3$-clique in $\G$, which is a complete $3$-subgraph, will be filled with a $2$-simplex (triangle). Then, in this sequence of simplicial complexes, we can systematically keep track of the evolution of the topological patterns. A $k$-dimensional topological feature (or $k$-hole) represent connected components ($0$-hole), loops ($1$-hole) and cavities ($2$-hole). For each $k$-hole $\sigma$, PH records its first appearance (birth) in the filtration sequence, say $\wh{\G}_{b_\sigma}$, and disappearance (death) in later complexes, $\wh{\G}_{d_\sigma}$ with a unique pair $(b_\sigma, d_\sigma)$, where $1\leq b_\sigma<d_\sigma\leq N$. PH records all these birth and death times of the topological features in \textit{persistence diagrams} (PDs). Let $0\leq k\leq D$ where $D$ is the highest dimension in the simplicial complex $\wh{\G}_N$. Then the $k$th persistence diagram ${\rm{PD}_k}(\G)=\{(b_\sigma, d_\sigma) \mid \sigma\in H_k(\wh{\G}_i) \mbox{ for } b_\sigma\leq i<d_\sigma\}$. Here, $H_k(\wh{\G}_i)$ represents the $k^{th}$ homology group of $\wh{\G}_i$ which keeps the information of the $k$-holes in the simplicial complex $\wh{\G}_i$. For the sake of notation, we skip the dimension (subscript $k$). With the intuition that the topological features with a long life span (persistent features) describe the hidden shape patterns in the data, these PDs provide a unique topological fingerprint of the graph $\G$.

As one can easily notice, the most important step in the PH machinery is the construction of the nested sequence of subgraphs $\G_1 \subseteq \ldots \subseteq \G_N=\G$. For a given unweighted graph $\G=(\V,\E)$ with $\V = \{v_1, \dots, v_{N}\}$ the set of nodes and $\E \subset \{\{v_i, v_j\} \in \V \times \V, i\neq j\}$ the set of edges, the most common technique is to use a filtering function $f:\mathcal{V}\to\R$ with a choice of thresholds $\I=\{\alpha_i\}$ where $\alpha_1=\min_{v \in \V} f(v)<\alpha_2<\ldots<\alpha_N=\max_{v \in \V} f(v)$. For $\alpha_i\in \I$, let $\V_i=\{v_r\in\V\mid f(v_r)\leq \alpha_i\}$. Let $\G_i$ be the induced subgraph of $\G$ by $\V_i$, i.e., $\G_i=(\V_i,\E_i)$ where $\E_i=\{e_{rs}\in \E\mid v_r,v_s\in\V_i\}$. This process yields a nested sequence of subgraphs $\G_1\subset \G_2\subset \ldots \subset\G_N=\G$, called \textit{the sublevel filtration} induced by the filtering function $f$. We denote PDs obtained via sublevel filtration for a filtering function $f$ as $\PD(\G,f)$. The choice of $f$ is crucial here, and in most cases, $f$ is either an important function from the domain of the data, e.g., atomic number for chemical compounds, or a function defined from intrinsic properties of the graph, e.g., degree and betweenness. Similarly, for a weighted graph, one can use sublevel filtration on the weights of the edges and obtain a suitable filtration reflecting the domain information stored in the edge weights. 

\subsection{Single Persistence Vectorizations} \label{sec:SPvec}
 
While PH extracts hidden shape patterns from data as persistence diagrams (PD), PDs being collections of points in $\R^2$ by itself are not very practical for statistical and machine learning purposes. Instead, the common technique is to faithfully represent PDs as kernels~\cite{kriege2020survey} or vectorizations~\cite{ali2022survey}. This provides a practical way to use the outputs of PH in real-life applications. \textit{Single Persistence Vectorizations} transform obtained PH information (i.e., PDs) into a function or a feature vector form which are much more suitable for machine learning tools than PDs. Common single persistence (SP) vectorization methods are Persistence Images~\cite{adams2017persistence}, Persistence Landscapes~\cite{Bubenik:2015}, Silhouettes~\cite{chazal2014stochastic}, and various Persistence Curves~\cite{chung2019persistence}. These vectorizations define a single variable or multivariable functions out of PDs, which can be used as fixed size $1D$ or $2D$ vectors in applications, i.e. $1\times n$ vectors or $m\times n$ vectors. For example, a Betti curve for a PD with $n$ thresholds can also be expressed as $1\times n$ size vectors. Similarly, Persistence Images is an example of $2D$ vectors with the chosen resolution (grid) size.

\subsection{Multidimensional Persistence} \label{sec:MP}

MultiPersistence introduces a novel concept that holds the potential to significantly enhance the performance of single-parameter persistence. The term \textit{single} is applied because we filter the data in just one direction, $\G_1\subset \dots\subset\G_N=\G$. The filtration's construction is pivotal in achieving a detailed analysis of the data to uncover concealed patterns. When utilizing a single function $f:\V\to\R$ containing crucial domain information (e.g., value for blockchain networks, atomic number for protein networks), it induces a single-parameter filtration as described earlier.

However, numerous datasets possess multiple highly relevant domain functions for data analysis. Employing these functions concurrently would yield a more comprehensive grasp of the concealed patterns. This insight led to the suggestion of the MP theory as a natural extension of single persistence (SP).

In simpler terms, if we use only one filtering function, sublevel sets induce a single parameter filtration $\wh{\G}_1\subset \dots\subset\wh{\G}_N=\wh{\G}$. Instead, if we use two or more functions, then it would give a way to study the data in a much finer resolution. For example, if we have two node functions $f:\V\to\R$ and $g:\V\to\R$ with valuable complementary information of the network, MP is presumed to produce a unique topological fingerprint combining the information from both functions. These pair of functions  $f,g$ induce a multivariate filtering function $F:\V \to \R^2$ with $F(v)=(f(v),g(v))$. Again, we can define a set of non-decreasing thresholds $\{\alpha_i\}_1^m$ and $\{\beta_j\}_1^n$ for $f$ and $g$ respectively. Then, $\V_{ij}=\{v_r\in \V\mid f(v_r)\leq \alpha_i , g(v_r)\leq\beta_j\}$, i.e., $\V_{ij}=F^{-1}((-\infty,\alpha_i]\times(-\infty,\beta_j])$. Then, let $\G_{ij}$ be the induced subgraph of $\G$ by $\V_{ij}$, i.e., the smallest subgraph of $\G$ generated by $\V_{ij}$. This induces a bifiltration of complexes $\{\wh{\G}_{ij}\mid 1\leq i\leq m, 1\leq j\leq n\}$. We can imagine $\{\wh{\G}_{ij}\}$ as a rectangular grid of size $m\times n$ (See \cref{Fig:ToyMP}).  

By computing the homology groups of these complexes, $\{H_k(\wh{\G}_{ij})\}$, we obtain the induced bi-graded persistence module (a rectangular grid of size $m\times n$). Again, the idea is to keep track of the $k$-dimensional topological features via the homology groups $\{H_k(\wh{\G}_{ij})\}$ in this grid. As we explained in \cref{sec:MP_theory}, because of the technical issues related to commutative algebra, converting the multipersistence module into a mathematical representation like a \textquote{Multipersistence Diagram} is unfeasible. As a  result, we do not have an effective vectorization of the MP module. These technical obstacles prevent this promising approach to reach its full potential in real-life applications. 

In this paper, we overcome this problem by producing practical and computationally efficient vectorizations by utilizing the slicing idea in the multipersistence grid in a structured way, as we describe next. 
 
\section{Effective Multidimensional Persistence Summaries} 
\label{sec:EMP2}

We now introduce our Effective MultiPersistence framework which describes a way to expand most single persistence vectorizations (\cref{sec:SPvec}) as multidimensional vectorizations by utilizing the MP approach. In particular, by using the existing single-parameter persistence vectorizations, we produce multidimensional vectors by effectively using one of the filtering directions as the \textit{slicing direction} in the multipersistence module. We give the basic setup in this section, and the generalization directions and further details in \Cref{sec:generalEMP}.

\subsection{EMP Framework} \label{sec:EMP}

To keep the exposition simple, we describe our framework for $2$-parameter multipersistence ($d=2$). For $d>2$, the construction is similar and given below. The outline is as follows: For   two given filtering functions $f,g$ for a graph $\G$, we use the first function $f$ to get a single parameter filtering of the data, i.e., $\G_1 \subseteq \ldots \subseteq \G_m=\G$. Then, we use the second function in each subgraph $\G_i$ to obtain persistence diagram $\PD(\G_i, g)$ for $1\leq i\leq m$. Hence, we obtain $m$ persistence diagrams $\{\PD(\G_i,g)\}_{i=1}^m$. Next, by applying the chosen SP vectorization $\varphi$ to each PD, we obtain $m$ different same length vector $\vec{\varphi}(\PD(\G_i,g))$, say $1\times k$ vector. By combining all $m$ $1D$-vectors, we obtain EMP vectorization $\M_\varphi$ with $\M_\varphi^i=\vec{\varphi}(\PD(\G_i,g))$ where $\M_\varphi^i$ represents the $i^{th}$ row of $\M_\varphi$ which is a $2D$-vector (matrix) of size $m\times k$ (See \cref{Fig:Framework}).

Here, we give the details for sublevel filtration for two filtering functions. In \Cref{sec:otherfiltration}, we explain how to modify the construction for weight filtrations or Vietoris-Rips filtrations. Let $\G=(\V,\E)$ be a graph, and let $f,g:\V\to \R$ be two filtering functions with threshold sets $\{\alpha_i\}_{i=1}^m$ and $\{\beta_j\}_{j=1}^n$ respectively.  Let $\V_i=\{v_r\in\V\mid f(v_r)\leq \alpha_i\}$. Let $\G_i$ be the induced subgraph of $\G$ by $\V_i$. This gives a filtering of the graph (nested subgraphs) as  $\G_1 \subseteq \ldots \subseteq \G_m=\G$. Recall that $g:\V\to \R$ is another filtering function for $\G$. Now, we fix $1\leq i_0\leq m$, and consider $\G_{i_0}$. By restricting $g$ on $\G_{i_0}$, we get persistence diagram $\PD(\G_{i_0}, g)$  as follows. Let $\V_{i_0j}=\{v_r\in\V_{i_0}\mid f(v_r)\leq \beta_j\}$, and let $\G_{i_0j}$ be the induced subgraph of $\G_{i_0}$ by $\V_{i_0j}$. This defines a finer filtering of the graph $\G_{i_0}$ as $\G_{i_01} \subseteq \G_{i_02} \ldots \subseteq \G_{i_0n}=\G_{i_0}$ (See \Cref{Fig:ToyMP} in the appendix). Corresponding clique complexes defines a filtration $\wh{\G}_{i_01} \subseteq \wh{\G}_{i_02} \ldots \subseteq \wh{\G}_{i_0n}=\wh{\G}_{i_0}$. This filtration gives the persistence diagram $\PD(\G_{i_0}, g)$. Hence, for each $1\leq i\leq m$, we obtain a persistence diagram $\PD(\G_i,g)$.  

The next step is to use vectorization on these $m$ persistence diagrams. Let $\varphi$ be a single persistence vectorization, e.g., Persistence Landscape, Silhouette, Entropy, Betti, Persistence Image. By applying the chosen SP vectorization $\varphi$ to each PD, we obtain a function $\varphi_i=\varphi(\PD(\G_i,g))$ where in most cases it is a single variable function on the threshold domain $[\beta_1,\beta_n]$, i.e., $\varphi_i:[\beta_1,\beta_n]\to \R$. For the multivariable case (e.g., Persistence Image), we give an explicit description in the examples section below. Most such vectorizations being induced from a discrete set of points $\PD(\G)$, they naturally can be expressed as a $1D$ vector of length $k$. In the examples below, we explain this conversion in detail. Then, let $\vec{\varphi}_i$ be the corresponding $1\times k$ vector for the function $\varphi_i$. Now, we are ready to define our EMP summary $\M_\varphi$ which is a $2D$-vector (a matrix) 
$$\M_\varphi^i=\vec{\varphi}_i \quad \mbox{for} \quad 1\leq i\leq m,$$
where $\M_\varphi^i$ is the $i^{th}$-row of $\M_\varphi$. Hence, $\M_\varphi$ is a $2D$-vector of size $m\times k$. Each row $\M_\varphi^i$ is the vectorization of the persistence diagram $\PD(\G_i, g)$ via the SP vectorization method $\varphi$. We use the first filtering function $f$ to get a finer look at the graph as $\G_1 \subseteq \ldots \subseteq \G_m=\G$. Then, we use the second filtering function $g$ to obtain the persistence diagrams $\PD(\G_i,g)$ of these finer pieces. In a way, we look at $\G$ with a two-dimensional resolution $\{\G_{ij}\}$ and we keep track of the evolution of topological features in the induced bifiltration  $\{\wh{\G}_{ij}\}$. The main advantage of this technique is that the outputs are fixed-size matrices (or arrays) for each dataset which is highly suitable for various machine learning models.

\begin{figure}[t!]
    \centering
    \includegraphics[width=.99\textwidth, angle =0]{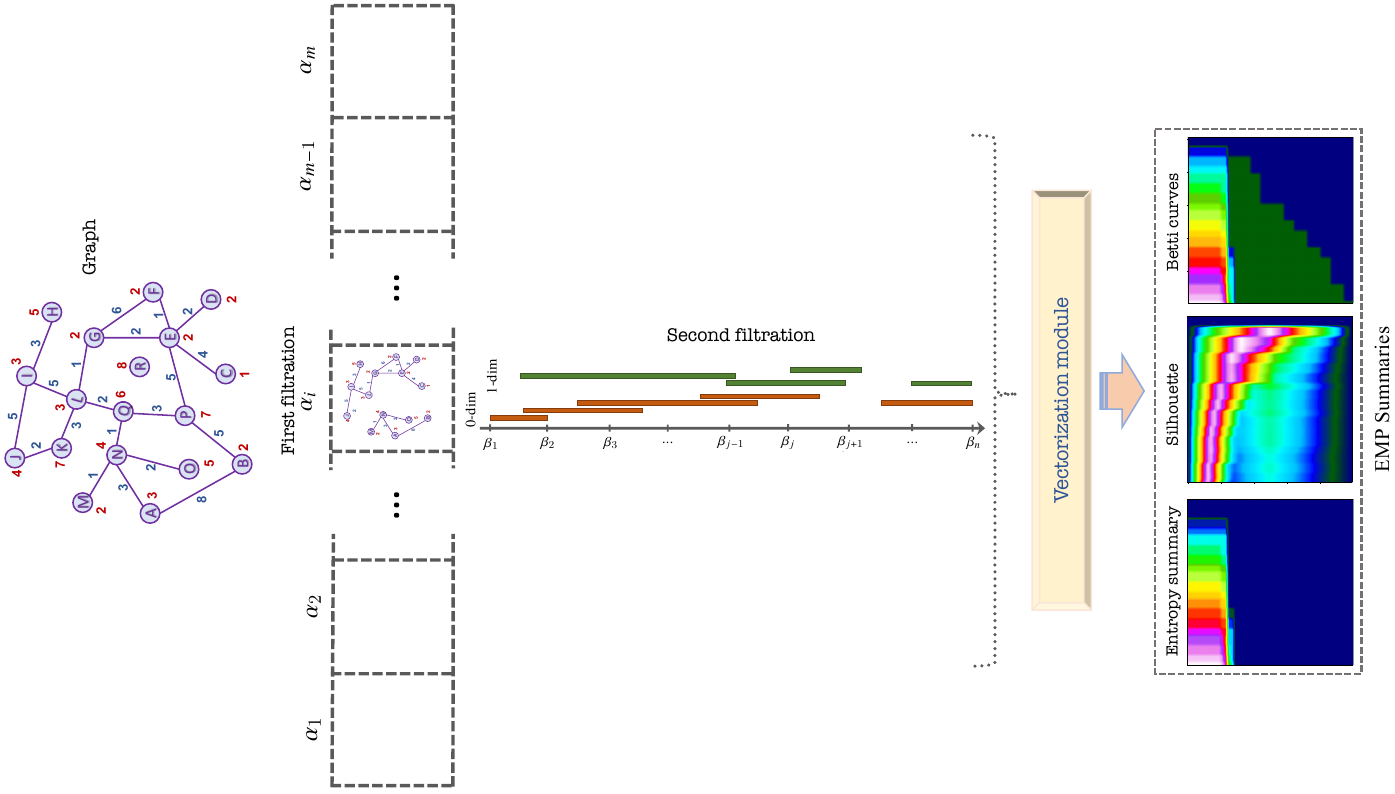}
    \caption{Illustration of the EMP framework for networks. Using the pair of filtering functions $f$, $g$ we define non-decreasing thresholds $\{\alpha_i\}_1^m$ and $\{\beta_j\}_1^n$, respectively, based on node features, red, and edge features, blue. Both, filtrations and vectorizations run in parallel to better use computational resources and produce EMP representations in a timely manner. 
    \label{Fig:Framework}}
\end{figure}

\smallskip
\noindent {\bf Order of the filtering functions.} \quad Since we only use horizontal slices, the first function is only used for finer filtration, and the second function gives the persistence diagrams. This makes our method a-symmetric (the order is important). Hence, one can change the order and get different performance results for the model. We discuss the order choice in detail and give experiments in \Cref{SecApx:AS:OrderFilt}.

\smallskip
\noindent {\bf Higher Dimensional Multipersistence.} \quad Similarly, for $d=3$, let $f,g,h$ be filtering functions and let $\{\G_{ij}\}$ be the bifiltering of the data, e.g., sublevel filtration for two functions $f, g$. Then, again by using the third function $h$, we find $m\cdot n$ persistence diagrams $\{\PD(\G_{ij},h)\}_{i,j=1}^{m,n}$. Similarly, for a given SP vectorization $\varphi$, one obtains a $1D$-vector $\vec{\varphi}(\PD(\G_{ij},g))$ of size $1\times k$ for each $i,j$. This produces a $3D$-vector (array) $\M_\varphi$ of size $m\times n\times k$ where $\M_\varphi^{ij}=\vec{\varphi}(\PD(\G_{ij},g))$. For $d>3$, one could follow a similar route.

\subsection{Examples of EMP Summaries} \label{sec:EMPexamples}

Here, we discuss explicit constructions of EMP summaries for most common SP vectorizations. As noted above, the framework is very general and it can be applied to most SP vectorization methods. In all the examples below, we use the following setup: Let $\G=(\V,\E)$ be a graph, and let $f,g:\V\to \R$ be two filtering functions with threshold sets $\{\alpha_i\}_{i=1}^m$ and $\{\beta_j\}_{j=1}^n$ respectively. As explained above, we first apply sublevel filtering with $f$ to get a sequence of nested subgraphs,  $\G_1 \subseteq \ldots \subseteq \G_m=\G$. Then, for each $\G_i$, we apply sublevel filtration with $g$ to get persistence diagram $\PD(\G_i,g)$. Therefore, we will have $m$ PDs. In the examples below, for a given SP vectorization $\varphi$, we explain how to obtain a vector $\vec{\varphi}(\PD(\G_i,g))$, and define the corresponding EMP $\M_\varphi$. Note that we skip the homology dimension (subscript $k$ for $PD_k(\G)$) in the discussion. In particular, for each dimension $k=0,1,\dots$, we will have one EMP $\M_\varphi(\G)$ (a matrix or array) corresponding to $\{\vec{\varphi}(PD_k(\G_i,g))\}$. The most common dimensions are $k=0$ and $k=1$ in applications. Recently, \citet{demir2022todd} has successfully applied a similar vectorization in drug discovery problem.

\smallskip
\noindent {\bf EMP Landscapes.} \quad Persistence Landscapes $\lambda$ are one of the most common SP vectorizations introduced by~\cite{Bubenik:2015}. For a given persistence diagram $\PD(\G)=\{(b_i,d_i)\}$, $\lambda$ produces a function $\lambda(\G)$ by using generating functions $\Lambda_i$ for each $(b_i,d_i)\in \PD(\G)$, i.e., $\Lambda_i:[b_i,d_i]\to\R$ is a piecewise linear function obtained by two line segments starting from $(b_i,0)$ and $(d_i,0)$ connecting to the same point $(\frac{b_i+d_i}{2},\frac{d_i-b_i}{2})$. Then, the \textit{Persistence Landscape} function $\lambda(\G):[\e_1,\e_q]\to\R$ for $t\in [\e_1,\e_q]$ is defined as 
$$\lambda(\G)(t)=\max_i\Lambda_i(t)$$
where $\{\e_k\}_1^q$ are thresholds for the filtration used.

Considering the piecewise linear structure of the function, $\lambda(\G)$ is completely determined by its values at $2q-1$ points, i.e., $\frac{b_i\pm d_i}{2}\in\{\e_1, \e_{1.5}, \e_2, \e_{2.5}, \dots ,\e_q\}$ where $\e_{k.5}={(\e_k+\e_{k+1})}/{2}$. Hence, a vector of size $1\times (2q-1)$ whose entries the values of this function would suffice to capture all the information needed, i.e.
$\vec{\lambda}=[ \lambda(\e_1)\ \lambda(\e_{1.5})\ \lambda(\e_2)\ \lambda(\e_{2.5})\ \lambda(\e_3)\ \dots \ \lambda(\e_q)]$

Considering we have threshold set $\{\beta_j\}_{j=1}^n$ for the second filtering function $g$, $\vec{\lambda}_i=\vec{\lambda}(\PD(\G_i,g))$ will be a vector of size $1\times 2n-1$. Then, as $\M_\lambda^i=\vec{\lambda}_i$ for each $1\leq i\leq m$, EMP Landscape $\M_\lambda(\G)$ would be a $2D$-vector (matrix) of size $m\times (2n-1)$.

\smallskip
\noindent {\bf EMP Surfaces.} \quad
Next, we give an important family of SP vectorizations, Persistence Curves~\cite{chung2019persistence}. This is an umbrella term for several different SP vectorizations, i.e., Betti Curves, Life Entropy, Landscapes, et al. Our EMP framework naturally adapts to all Persistence Curves to produce multidimensional vectorizations. As Persistence Curves produce a single variable function in general, they all can be represented as $1D$-vectors by choosing a suitable mesh size depending on the number of thresholds used. Here, we describe one of the most common Persistence Curves in detail, i.e., Betti Curves. It is straightforward to generalize the construction to other Persistence Curves.

Betti curves are one of the simplest SP vectorizations as it gives the count of the topological features at a given threshold interval. In particular, $\beta_k(\Delta)$ is the total count of $k$-dimensional topological feature in the simplicial complex $\Delta$, i.e., $\beta_k(\Delta)=rank(H_k(\Delta))$ (See \cref{Fig:ToyMP} in the Appendix). In particular, $\beta_k(\G):[\e_1,\e_{q+1}]\to\R$ is a step function defined as 
$\beta_k(\G)(t)=rank(H_k(\wh{\G}_i))$
for $t\in [\e_i,\e_{i+1})$, where $\{\e_i\}_1^q$ represents the thresholds for the filtration used. Considering this is a step function where the function is constant for each interval $[\e_i,\e_{i+1})$, it can be perfectly represented by a vector of size $1\times q$ as $\vec{\beta}(\G)=[ \beta(1)\ \beta(2)\ \beta(3)\ \dots \ \beta(q)]$. 

Then, with the threshold set $\{\beta_j\}_{j=1}^n$ for the second filtering function $g$, $\vec{\beta}_i=\vec{\beta}(\PD(\G_i,g))$ will be a vector of size $1\times n$. Then, as $\M_\beta^i=\vec{\beta}_i$ for each $1\leq i\leq m$, \textit{EMP Betti Summary} $\M_\beta(\G)$ would be a $2D$-vector (matrix) of size $m\times n$ (\cref{Fig:EMP}). In particular, each entry $\M_\beta=[m_{ij}]$ is just the Betti number of the corresponding clique complex in the bifiltration $\{\wh{\G}_{ij}\}$, i.e., $m_{ij}=\beta(\wh{\G}_{ij}).$ This matrix $\M_\beta$ is also called bigraded Betti numbers in the literature, and computationally much faster than other vectorizations~\cite{lesnick2022computing}.

We give further EMP examples (\textit{EMP Silhouettes} and \textit{EMP Images}) in \Cref{sec:further_examples}.

\begin{figure*}[t!]
    \centering
    \includegraphics[width=1.0\textwidth]{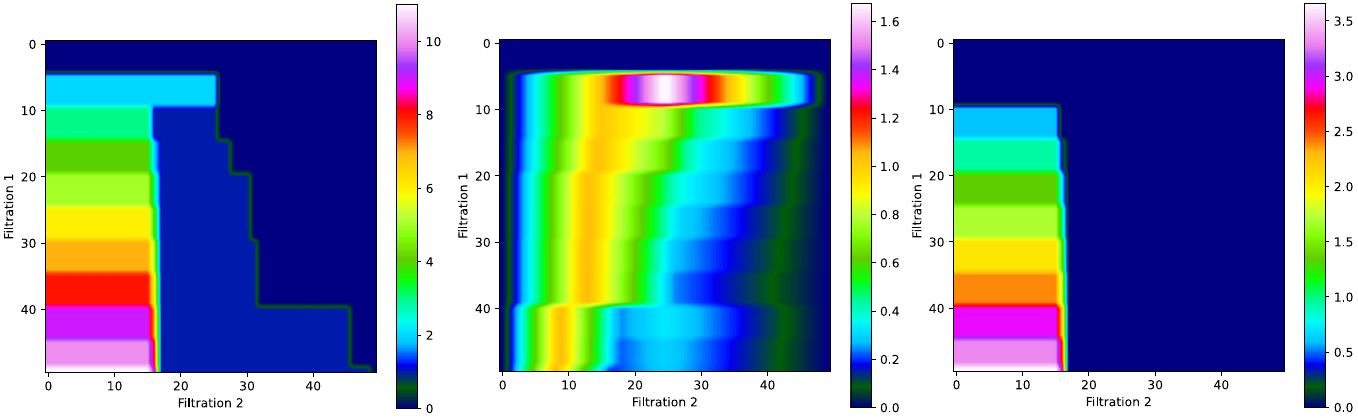}
    \caption{For the same network and the same filtering functions, EMP Betti Summary (left), EMP Silhouette (center), and EMP Entropy Summary (right) can produce highly different topological summaries emphasizing different information in persistence diagrams. \label{Fig:EMP}}
\end{figure*}

\subsection{Stability of EMP Summaries}  \label{sec:stability-main}

We now demonstrate that when the original single-parameter vectorization $\varphi$ is stable, the resulting EMP vectorization $\M_\varphi$ also maintains stability. The specifics of the stability concept in persistence theory are outlined, along with examples of stable SP vectorizations, in \cref{sec:stability2}.

We give generalizations of these stability notions and metrics in multidimensions and the proof of the following stability theorem in \Cref{sec:stability}.

\begin{theorem} \label{thm:stabilityMain}
Let $\varphi$ be a stable SP vectorization. Then, the induced EMP Vectorization $\M_\varphi$ is also stable, i.e., with the notation given in \Cref{sec:stability}, there exists $\wh{C}_\varphi>0$ such that for any pair of graphs $\G^+$ and $\G^-$, we have the following inequality.
$$\mathfrak{D}(\M_\varphi(\G^+),\M_\varphi(\G^-))\leq \wh{C}_\varphi\cdot \mathbf{D}_{p_\varphi}(\{\PD(\G^+)\},\{\PD(\G^-)\}).$$
\end{theorem}

\section{Experiments}

In our experiments, we used nine widely used benchmark datasets in graph classification tasks. 
  In particular, we consider (i) three molecule graphs~\citep{sutherland2003spline}: BZR\_MD, COX2\_MD, and DHFR\_MD; (ii) two biological graphs~\citep{kriege2012subgraph,borgwardt2005protein}: MUTAG and PROTEINS; and (iii) four social graphs: IMDB-Binary (IMDB-B), IMDB-Multi (IMDB-M), REDDIT-Binary (REDDIT-B), and REDDIT-Multi-5K (REDDIT-M-5K). The dataset statistics are given in appendix \Cref{{tab:datasets}}. 
 
\subsection{Experimental Setup}\label{Sec:Experimental_Setup}
To assess the effectiveness of our EMP summaries in graph representation learning, we assess them using the random forest (RF) classifier in a graph classification task. We select RF to underline EMP's adaptability, although EMP can seamlessly integrate with advanced DL models as a trainable component.

For the RF classifier, we fix the forest's tree count at 1000, the minimum sample requirement at 2, and the Gini impurity for split quality measurement. All EMP representations we introduce are vectorized.
 
We apply filtrations based on the available information of each dataset, either using their graph structure or their provided node/edge features. Our pool of filtering functions include: atomic weight, closeness, edge-betweenness, weighted degree, Katz centrality, and Ricci curvatures. We also use power filtration as the last filtration direction. To test our EMP framework we use three vectorizations: Betti curves, Silhouettes and Entropy Summary functions, thus producing EMP matrix representations that can be embedded in classic and modern machine learning algorithms.  

We give further details on our experimental setup in \Cref{SecApx:DataExperiment}. Furthermore,  \cref{SecApx:CompComplexity} includes an analysis of computational complexity. The source code is available in Python\footnote{\url{https://www.dropbox.com/scl/fo/5kydyx2ivu1vpqi8hd0ob/h?rlkey=5u46k0x6p4hewpq5lhk9zpxl5&dl=0}}.

\begin{table*}[t!]
\centering
\caption{\textbf{Accuracy.} Classification accuracy (in \% $\pm$ standard deviation) of EMP summary on nine benchmark datasets. The best results are in {\bf bold} font and the second best results are marked {\it underlined}. \label{classification_results_MP_on_graphs}}
\vspace{.1in}
\resizebox{1\linewidth}{!}{
\setlength\tabcolsep{3pt}
\begin{tabular}{lccccccccc}
\toprule
\textbf{{Model}} &\textbf{{BZR\_MD}} & \textbf{{COX2\_MD}} & \textbf{{DHFR\_MD}}& \textbf{{MUTAG}} & \textbf{{PROTEINS}} &\textbf{{IMDB-B}}&\textbf{{IMDB-M}}&\textbf{{REDDIT-B}}&\textbf{{REDDIT-5K}} \\
\midrule
CSM~\cite{kriege2012subgraph} &\underline{77.63{\footnotesize $\pm$1.29} }&OOT &OOT &87.29{\footnotesize $\pm$1.25 }&OOT&OOT&OOT&OOT &OOT \\
H-SP~\cite{morris2016faster}& 60.08{\footnotesize $\pm$0.88 }& 59.92{\footnotesize $\pm$0.66 }& 67.95{\footnotesize $\pm$0.00 }& 80.90{\footnotesize $\pm$0.48 }&74.53{\footnotesize $\pm$0.35}&73.34{\footnotesize $\pm$0.47}& {\bf 51.58{\footnotesize $\pm$0.42} }& OOM &OOM \\
H-WL~\cite{morris2016faster}&52.64{\footnotesize $\pm$1.20 }& 57.15{\footnotesize $\pm$1.20 }& 66.08{\footnotesize $\pm$1.02}&75.51{\footnotesize $\pm$1.34 }& 74.53{\footnotesize $\pm$0.35}& 72.75{\footnotesize $\pm$1.02}& \underline{50.73{\footnotesize $\pm$0.63} }& OOM &OOM \\
MLG~\cite{kondor2016multiscale}& 51.46{\footnotesize $\pm$0.61}& 51.15{\footnotesize $\pm$0.00}& 67.95{\footnotesize $\pm$0.00 }& 78.53{\footnotesize $\pm$2.25 }& \underline{75.55{\footnotesize $\pm$0.71}}&52.56{\footnotesize $\pm$0.42}& 34.27{\footnotesize $\pm$0.33 }& OOM &OOM \\
WL~\cite{shervashidze2011weisfeiler}& 67.45{\footnotesize $\pm$1.40}& 60.07{\footnotesize $\pm$2.22}& 62.56{\footnotesize $\pm$1.51 }& 85.75{\footnotesize $\pm$1.96 }&  73.06{\footnotesize $\pm$0.47}&71.15{\footnotesize $\pm$0.47 }& 50.25{\footnotesize $\pm$0.72 }& 77.95{\footnotesize $\pm$0.60 }&  51.63{\footnotesize $\pm$0.37}\\
WL-OA~\cite{kriege2016valid} &68.19{\footnotesize $\pm$1.09}& 62.37{\footnotesize $\pm$2.11}& 64.10{\footnotesize $\pm$1.70}& 86.10{\footnotesize $\pm$1.95 }& 73.50{\footnotesize $\pm$0.87}&\underline{74.01{\footnotesize $\pm$0.66} }& 49.95{\footnotesize $\pm$0.46 }& 87.60{\footnotesize $\pm$0.33 }&OOM \\
{FC-V}~\cite{o2021filtration}&75.61{\footnotesize $\pm$1.13 }&\textbf{73.41{\footnotesize $\pm$0.79} }&\underline{76.78{\footnotesize $\pm$0.69}  }& \underline{87.31{\footnotesize $\pm$0.66} }& 74.54{\footnotesize $\pm$0.48 }&73.84{\footnotesize $\pm$0.36 }& 46.80{\footnotesize $\pm$0.37 }& 89.41{\footnotesize $\pm$0.24 }& 52.36{\footnotesize $\pm$0.37}\\
GNNs~\cite{erricafair}  &69.87{\footnotesize $\pm$1.29}& 66.05{\footnotesize $\pm$3.16}& 73.11{\footnotesize $\pm$1.59 }& 80.42{\footnotesize $\pm$2.07 }& \textbf{75.80{\footnotesize $\pm$3.70}}&71.20{\footnotesize $\pm$3.90 }& 49.10{\footnotesize $\pm$3.50 }& \underline{89.90{\footnotesize $\pm$1.90} }& {\bf 56.10{\footnotesize $\pm$1.60}}\\
\midrule
\textbf{EMP} & {\bf 77.76{\footnotesize $\pm$0.95}} & \underline{70.12{\footnotesize $\pm$0.81}} & {\bf 80.13{\footnotesize $\pm$0.94}} & {\bf 88.79{\footnotesize $\pm$0.63}}  & 72.78{\footnotesize $\pm$0.54  }& {\bf 74.44{\footnotesize $\pm$0.45} } & 48.01{\footnotesize $\pm$0.42}  & {\bf 91.03{\footnotesize $\pm$0.22} } & \underline{54.41{\footnotesize $\pm$0.32} } \\
\bottomrule
\end{tabular}}
\end{table*}

\subsection{Experimental Results}

We compare our model with three types of state-of-the-art baselines, covering six graph kernels, six graph neural networks (GNNs), and one topological method. Graph kernels: (i) comprised of the subgraph matching kernel (CSM)~\cite{kriege2012subgraph}, (ii) Shortest
Path Hash Graph Kernel (HGK-SP)~\cite{morris2016faster}, (iii) WL Hash
Graph Kernel (HGK-WL)~\cite{morris2016faster}, (iv) Multiscale Laplacian
Graph Kernel (MLG)~\cite{kondor2016multiscale}, (v) Weisfeiler–
Lehman (WL)~\cite{shervashidze2011weisfeiler}, and (vi) WL Optimal Assignment (WL-OA)~\cite{kriege2016valid}; topological method: filtration curves (FC-V)~\cite{o2021filtration}, six graph neural networks including GCN, DGCNN, Diffpool, ECC, GIN, GraphSage which are compared in~\cite{erricafair}. We report the best results of these six GNNs in the GNNs row in \Cref{classification_results_MP_on_graphs}. For each dataset, we report our best performing model (\Cref{tab:ablation}). For all methods, we report the average accuracy of 10 runs of 10-fold CV along with the standard deviation. 

Table~\ref{classification_results_MP_on_graphs} shows the results of different methods on nine graph datasets. Out-of-time (OOT) results indicate that a method could not complete the classification task within 12 hours, and OOM means ``out-of-memory'' (from an allocation of 128 GB RAM). In \Cref{tab:MPcomparison}, we further compare our EMP model with other existing MultiPersistence vectorizations in the literature. Namely, MP landscapes (MP-L)~\cite{vipond2020multiparameter}, MP Images (MP-I)~\cite{carriere2020multiparameter}, multiparameter persistence kernel (MP-K)~\cite{corbet2019kernel}, the generalized rank invariant landscape (GRIL)~\cite{xin2023gril}, and MP Hilbert and Euler characteristic functions (MP-H and MP-E)~\cite{oudot2021stability} and MP Signed Barcodes (MP-SB)~\cite{loiseaux2023stable} where we reported the best of the four different vectorizations namely, MP-HSM-C,	MP-ESM-C, MP-ESM-SW, MP-HSM-SW.

We observe the following phenomena:
\begin{itemize}
    \item[$\diamond$]  Compared with all baselines, out of 9 benchmark datasets, the proposed EMP summaries achieve the best performance on 5 datasets (BZR\_MD, DFHR\_MD, MUTAG, IMDB-B,  REDDIT-B), and second best in 2 datasets (COX2\_MD, REDDIT-5K).  
    
    \item[$\diamond$]  EMP summary consistently outperforms Filtration Curves on all datasets except COX2\_MD and PROTEINS, indicating that the multiparameter structure of the EMP summaries can better capture the complex structural properties and local topological information in heterogeneous graphs.

    \item[$\diamond$]  When compared with other Multipersistence Vectorizations, EMP consistently ranks among top two performances.
    
   \item[$\diamond$] Moreover, EMP Summaries consistently deliver competitive results with GNNs and kernel methods in all benchmark datasets.
     This indicates that EMP summary introduces powerful topological and geometric schemes for node features and graph representation learning.
\end{itemize}

\begin{table*}[h!]
\centering
\caption{Comparison with other Multipersistence Vectorizations \label{tab:MPcomparison}} 
\resizebox{1\columnwidth}{!}{
\setlength\tabcolsep{5pt}
\begin{tabular}{lccccccc|c}
\toprule
\textbf{Dataset} &	\textbf{MP-Kernel}	& \textbf{\footnotesize MP-Landscape} &	\textbf{MP-Images} &	\textbf{GRIL} &	\textbf{MP-Hilbert}& \textbf{MP-Euler}	&	\textbf{MP-SB} & \textbf{EMP}\\
\midrule
COX2 & \underline{79.9{\footnotesize $\pm$1.8}}&	79.0{\footnotesize $\pm$3.3 }&	77.9{\footnotesize $\pm$2.7 }&	79.8{\footnotesize $\pm$2.9}&	78.2{\footnotesize $\pm$1.7 }&	77.3{\footnotesize $\pm$1.1}&		{78.4{\footnotesize $\pm$0.7}}&	 \textbf{79.9{\footnotesize $\pm$0.8}}\\
IMDB-B &	68.2{\footnotesize $\pm$1.2 }&	71.2{\footnotesize $\pm$2.0 }&	71.1{\footnotesize $\pm$2.1 }&	65.2{\footnotesize $\pm$2.6 }&	73.0{\footnotesize $\pm$4.5 }&	72.0{\footnotesize $\pm$1.9}&		\textbf{75.1{\footnotesize $\pm$3.4 }}&	 \underline{74.4{\footnotesize $\pm$0.5}}\\
IMDB-M &	46.9{\footnotesize $\pm$2.6 }&	46.2{\footnotesize $\pm$2.3 }&	46.7{\footnotesize $\pm$2.7 }&	NA	&49.1{\footnotesize $\pm$1.6 }&	\underline{50.0{\footnotesize $\pm$0.8}}&		\textbf{51.1{\footnotesize $\pm$1.3}}&	48.0{\footnotesize $\pm$0.4}\\
MUTAG &	86.1{\footnotesize $\pm$5.2 }&	84.0{\footnotesize $\pm$6.8 }&	85.6{\footnotesize $\pm$7.3 }&	87.8{\footnotesize $\pm$4.2 }&	87.2{\footnotesize $\pm$5.8 }&	87.2{\footnotesize $\pm$4.3 }&		\textbf{89.9{\footnotesize $\pm$4.3 }}&  \underline{88.8{\footnotesize $\pm$0.6}}\\
PROTEINS &	67.5{\footnotesize $\pm$3.1 }&	65.8{\footnotesize $\pm$3.3 }&	67.3{\footnotesize $\pm$3.5 }&	70.9{\footnotesize $\pm$3.1 }&	70.2{\footnotesize $\pm$2.1 }&	70.7{\footnotesize $\pm$1.9 }&		\textbf{73.9{\footnotesize $\pm$1.7 }}&  \underline{72.8{\footnotesize $\pm$0.5}}\\
\bottomrule
\end{tabular}}
\end{table*}

\subsection{Ablation Study}  \label{SecApx:AblationStudy}
To further investigate the effectiveness of the EMP vectorization function in graph representation learning, we have conducted ablation studies of various EMP summaries on the benchmark datasets.
In \Cref{tab:ablation}, we give the performances of different types of EMP summaries (i.e., EMP Silhouette (EMP-S), EMP Entropy (EMP-E), and EMP Betti (EMP-B) using only $0$-dimensional topological features ($H_0$), only $1$-dimensional topological features ($H_1$) and using both ($H_0+H_1$). Note that in these models, we also included  graph-based features (i.e., $f_\V$: node features, and $f_{\E}$: edge features; see more details in \cref{Sec:Experimental_Setup} and \cref{SecApx:DataExperiment}). 

These results suggest that: 
\begin{itemize}
\item[$\diamond$] The choice of the EMP summary can significantly affect the performance (e.g., BZR\_MD, DFHR\_MD),
\item[$\diamond$]  Using both dimensions ($H_0+H_1$) may not be better than using only one of the dimensions (e.g., PROTEINS, IMDB-B, REDDIT-5K)
\item[$\diamond$] The choice of topological dimension ($H_0$ or $H_1$) for EMP Summary can be crucial for some datasets (e.g. BZR\_MD, REDDIT-5K).
\end{itemize}

\begin{table*}[h!]
\centering
\caption{ {\bf Ablation Study.} Comparison of the performances of EMP summaries on nine benchmark datasets. The accuracies are given in \% ($\pm$ standard deviation) and {\bf bold numbers} represent the best results. \label{tab:ablation}}
\resizebox{1.\columnwidth}{!}{
\setlength\tabcolsep{3pt}
\begin{tabular}{lccccccccc}
\toprule
\textbf{{Model}} &\textbf{{BZR\_MD}} & \textbf{{COX2\_MD}} & \textbf{{DHFR\_MD}}& \textbf{{MUTAG}} & \textbf{{PROTEINS}} &\textbf{{IMDB-B}}&\textbf{{IMDB-M}}&\textbf{{REDDIT-B}}&\textbf{{REDDIT-5K}} \\
\midrule
EMP-B $H_0$ & 68.04$\pm$2.04  & 69.39$\pm$1.36 & 72.12$\pm$0.94 &88.56$\pm$0.66  & 71.54$\pm$0.43 & 73.13$\pm$0.44 & 46.33$\pm$0.43 & 90.52$\pm$0.20& \textbf{54.41$\pm$0.32} \\
EMP-B $H_1$ & 69.22$\pm$1.06  & \textbf{70.12$\pm$0.81} & 68.77$\pm$1.03 & 86.30$\pm$0.72 & 71.61$\pm$0.48  & \textbf{74.44$\pm$0.45}  & \textbf{48.01$\pm$0.42} & 88.23$\pm$0.28& 51.96$\pm$0.30 \\
EMP-B $H_0$ + $H_1$ & 69.08$\pm$1.47  & 69.27$\pm$1.15 & 74.21$\pm$0.91 & \textbf{88.79$\pm$0.63} & 71.52$\pm$0.53 & 73.20$\pm$0.36 &46.82$\pm$0.53 &\textbf{91.03$\pm$0.22} & 54.34$\pm$0.31  \\
\midrule
EMP-E $H_0$ & 65.29$\pm$1.75  &67.29$\pm$1.19  & 74.53$\pm$1.03 & 86.74$\pm$0.68  & \textbf{72.78$\pm$0.54}& 73.15$\pm$0.63 & 46.93$\pm$0.49 & 89.95$\pm$0.18& 53.47$\pm$0.36\\
EMP-E $H_1$ & {\bf 77.76$\pm$0.95}  &67.55$\pm$0.85  & 66.40$\pm$0.89 & 86.57$\pm$0.70 & 71.93$\pm$0.54 & 72.95$\pm$0.49 & 47.03$\pm$0.23  & 88.64$\pm$0.16&51.97$\pm$0.36 \\
EMP-E $H_0$ + $H_1$ & 73.32$\pm$1.09  &69.18$\pm$1.11  & 74.80$\pm$0.96 & 86.84$\pm$0.62 & 71.86$\pm$0.36 & 72.97$\pm$0.58 & 47.03$\pm$0.41 & 90.05$\pm$0.22 & 53.95$\pm$0.35\\
\midrule
EMP-S $H_0$ & 68.23$\pm$0.97  & 69.07$\pm$1.24 & 78.31$\pm$0.68 &87.89$\pm$0.57  & 71.51$\pm$0.60 & 73.83$\pm$0.39 & 47.38$\pm$0.53 & 88.42$\pm$0.24 & 52.29$\pm$0.19\\
EMP-S $H_1$ & 70.74$\pm$1.28  & 63.72$\pm$1.24 & 75.26$\pm$1.05 & 84.71$\pm$0.99 & 70.74$\pm$0.57 & 73.99$\pm$0.55 & 47.71$\pm$0.34 & 86.96$\pm$0.35 & 50.71$\pm$0.28\\
EMP-S $H_0$ + $H_1$ &  72.90$\pm$0.91 & 67.89$\pm$1.70 & \textbf{80.13$\pm$0.94 }& 88.10$\pm$0.83 & 71.61$\pm$0.56 &74.29$\pm$0.28  & 47.80$\pm$0.51  & 88.59$\pm$0.38& 52.75$\pm$0.22 \\
\bottomrule
\end{tabular}}
\end{table*}

\noindent {\bf Limitations:}
The main limitation of the EMP approach comes with the choice of filtering functions and pairing them. While these choices give flexibility to the model, the best function pairs to use may require domain information of the dataset. One way to bypass this issue is to use self-supervised methods to learn effective node and edge functions from the data.

\section{Conclusion}

We have introduced an innovative and computationally efficient summary approach for multidimensional persistence across various forms of data, with a specific focus on applications in graph-based machine learning. This novel framework, called Effective Multidimensional Persistence (EMP), offers a practical and effective method to integrate the concept of multidimensional persistence into real-world scenarios. The EMP approach seamlessly integrates with ML models, providing a unified enhancement to existing single persistent summaries.
In graph classification tasks, EMP summaries have demonstrated superior performance compared to state-of-the-art techniques across multiple benchmark datasets. Moreover, we have shown that EMP maintains important stability guarantees. This signifies a significant stride in bridging theoretical multipersistence concepts with the machine learning community, thus advancing the utilization of persistent homology in diverse contexts.
Looking ahead, our future endeavors aim to enrich the EMP framework by incorporating multiple slicing directions within the multipersistence grid. This involves leveraging deep learning methodologies to effectively combine outputs from various slicing directions.

\section*{Acknowledgements}

This work was partially supported by the National Science Foundation (NSF) under grants  DMS-2202584, DMS-2220613, OAC-1828467, DMS-1925346, CNS-2029661, DMS-2335846/2335847, TIP-2333703, OAC-2115094 Simons Foundation under grant \# 579977, a gift from Cisco Inc. and the grant from the Office of Naval Research (ONR) award N00014-21-1-2530. Part of this material is also based upon work supported by (while Y.R.G. was serving at)
the NSF. The views expressed in the article do not necessarily represent the views of NSF or ONR.
\bibliographystyle{unsrtnat}
\bibliography{emp}

\newpage
\setcounter{page}{1}

\appendix
\centerline{\Large \bf Appendix}

\medskip

Appendix sections give additional details for our experiments and methods. In \Cref{SecApx:Experiments}, we give further details on our experiments, including experimental setup, computational complexity and effects of order of filtration on EMP's performance. In \Cref{sec:generalEMP}, we explain how to generalize EMP framework to general data, and other type of filtration methods. We also discuss the difficulties in multipersistence theory in general, and our contribution in this context in \Cref{sec:MP_theory}. Finally, in \Cref{secapx:stab}, we prove our stability theorem.

\vspace{-.1in}

\section{Further Details on Experiments} \label{SecApx:Experiments}

\subsection{Experimental Setup} \label{SecApx:DataExperiment}

We vectorize our proposed EMP representations as input to RF. 
In our experiments, we use nine benchmark datasets for graph classification tasks (see Table~\ref{tab:datasets}). We have run our models for graph classification tasks on an 8-core DO droplet machine with Intel Xeon Scalable processors at a base frequency of 2.5 Ghz. 
\Cref{tab:datasets} summarizes the statistics of the datasets in our experiments.

\begin{table*}[h!]
\centering
\caption{Summary statistics of the datasets. \label{tab:datasets}}
\resizebox{.8\columnwidth}{!}{
\setlength\tabcolsep{5pt}
\begin{tabular}{lccccccc}
\toprule
\textbf{{Dataset}} & \textbf{{\# Graphs}} &\textbf{{Avg.} $|\mathcal{V}|$} & \textbf{{Avg.} $|\mathcal{E}|$} & \textbf{{\# Class}} &\textbf{{\# Node Attr.}} &\textbf{{\# Edge Attr.}}\\
\midrule
BZR\_MD &306 &21.30& 225.06 &2 &3 &-  \\ 
COX2\_MD &303 &26.28 &335.12 &2 &3 &-  \\
DHFR\_MD & 393 &23.87 &283.02 &2 &- &1 \\
MUTAG &188 &17.93 &19.79 &2 &- &- \\
PROTEINS &1113 &39.06 &72.82 &2 &1 &- \\
IMDB-B &1000 &19.77 &96.53 &2 &- & -\\
IMDB-M & 1500 & 13.00 & 65.94 & 3 &- &-\\
REDDIT-B & 2000 & 429.63 & 497.75 & 2 &- &-\\
REDDIT-5K & 4999 & 508.82 & 594.87 &5 &- &-\\
\bottomrule
\end{tabular}}
\end{table*}

The resolution of vectorization is the most significant parameter, which may impact the computational performance and results. As such, we use a fixed resolution to get consistent results in all experiments and consider time constraints on server usage. We use a resolution size of 50$\times$50 for each summary function, and the standard parameters set by the Gudhi library in Python~\footnote{\url{https://gudhi.inria.fr/python/latest/}}. The order of landscape summary function is set to $1$ ($\max$), whilst the power of weights is set to $1$ for silhouette summaries.

Results in Table~\ref{classification_results_MP_on_graphs} come from computing MP using filtering functions as follows. BZR\_MD, COX2\_MD, and DHFR\_MD use weighted node-degree and edge-power filtrations. PROTEINS uses node-closeness and edge-betweenness power filtrations. MUTAG, IMDB-BINARY, and IMDB-MULTI use node-Katz centrality and edge-Ricci curvature power filtrations. REDDIT-BINARY and REDDIT-MULTI-5K use node-Katz centrality and edge-Ricci curvature filtrations. We performed an empirical analysis to select previous filter functions for each graph network. We include the number of nodes and the number of edges as graph features. We use three types of vectorizations: Betti curves, silhouette functions, and entropy summary functions. For each case, we compute both 0-dim and 1-dim MP topological features.

\subsection{Computational Complexity} 
\label{SecApx:CompComplexity}

Computational complexity (CC) of persistence diagram $\PD_k(\Delta)$ is $\mathcal{O}(\mathcal{N}^3)$, where $\mathcal{N}$ is the number of $k$-simplices in $\Delta$~\citep{otter2017roadmap}. 
CC of EMP summary $\mathbf{M}_\varphi^d$ depends on the vectorization $\varphi$ used and the number $d$ of the filtering functions one uses.  
If $r$ is the resolution size of the multipersistence grid, then one needs $r^{(d-1)}$ single persistence diagrams to obtain $\mathbf{M}_\varphi^d$. Therefore,  $\mathrm{CC}(\mathbf{M}_\varphi^d)=\mathcal{O}(r^{(d-1)}\cdot \mathcal{N}^3 \cdot C_\varphi(m))$ where $C_\varphi(m)$ is CC for $\varphi$ and $m$ is the number of barcodes in $\PD_k$,  e.g., if $\varphi$ is persistence landscape, then $C_\varphi(m)=m^2$ and hence CC for EMP Landscape with $d=2$ is $\mathcal{O}(r\cdot \mathcal{N}^3\cdot m^2)$. In practice, $r$ is a constant and $m$ is small compared to $\mathcal{N}$, hence the complexity is again reduced to $\mathcal{O}(  \mathcal{N}^3)$. 
On the other hand, as Betti numbers do not need $\PD_k$ to be computed, it is possible to obtain much faster algorithms for EMP Betti Summary \citep{edelsbrunner2014computational}. Recently, \cite{lesnick2022computing} introduced a quite fast algorithm for EMP Betti summaries with $\mathcal{O}(\mathcal{M}^3)$ time where $\mathcal{M}$ is the rank of the multipersistence module with minimal representation.

\section{EMP Framework} \label{sec:generalEMP}

\subsection{Further EMP Examples} \label{sec:further_examples}

\noindent {\bf EMP Silhouettes.} \quad Silhouette~\cite{chazal2014stochastic} is another very popular SP vectorizations method in machine learning applications. The idea is similar to Persistence Landscapes, but this vectorization uses the life span of the topological features more effectively. For $\PD(\G)=\{(b_i,d_i)\}_{i=1}^N$, let $\Lambda_i$ be the generating function for $(b_i,d_i)$ as defined in Landscapes above. Then, \textit{Silhouette} function $\psi$ is defined as $$\psi(\G)=\dfrac{\sum_{i=1}^N w_i\Lambda_i(t)}{\sum_{i=1}^m w_i}, \ t\in[\e_1,\e_q],$$ where the weight $w_i$ is mostly chosen as the life span $d_i-b_i$, and $\{\e_k\}_{k=1}^q$ represents the thresholds for the filtration used. Again such a Silhouette function $\psi(\G)$ produces a $1D$-vector $\vec{\psi}(\G)$ of size $1\times (2q-1)$ as in Persistence Landscapes case. Similar to the EMP Landscapes, with the threshold set $\{\beta_j\}_{j=1}^n$ for the second filtering function $g$, $\vec{\psi}_i=\vec{\psi}(\PD(\G_i,g)$ will be a vector of size $1\times 2n-1$. Then, as $\M_\psi^i=\vec{\psi}_i$ for each $1\leq i\leq m$, EMP Landscape $\M_\psi(\G)$ would be again a $2D$-vector (matrix) of size $m\times (2n-1)$  (\cref{Fig:EMP}).

\smallskip

\noindent {\bf EMP Persistence Images.} \quad Another common SP vectorization is Persistence Images~\cite{adams2017persistence}. Different than most SP vectorizations, Persistence Images produces $2D$-vectors. The idea is to capture the location of the points in the persistence diagrams with a multivariable function by using the $2D$ Gaussian functions centered at these points. For $\PD(\G)=\{(b_i,d_i)\}$, let $\phi_i$ represent a $2D$-Gaussian centered at the point $(b_i,d_i)\in \R^2$. Then, one defines a multivariable function, \textit{Persistence Surface}, $\wt{\mu}=\sum_iw_i\phi_i$ where $w_i$ is the weight, mostly a function of the life span $d_i-b_i$. To represent this multivariable function as a $2D$-vector, one defines a $k\times l$ grid (resolution size) on the domain of $\wt{\mu}$, i.e., threshold domain of $\PD(\G)$. Then, one obtains the \textit{Persistence Image}, a $2D$-vector (matrix) $\vec{\mu}=[\mu_{rs}]$  of size $k\times l$ where $\mu_{rs}=\int_{\Delta_{rs}}\wt{\mu}(x,y)\,dxdy$ and $\Delta_{rs}$ is the corresponding pixel (rectangle) in the $k\times l$ grid.

This time, the resolution size $k\times l$ is independent of the number of thresholds used in the filtering, the choice of $k$ and $l$ is completely up to the user. Recall that by applying the first function $f$, we have the nested subgraphs $\{\G_i\}_{i=1}^m$. For each $\G_i$, the persistence diagram $\PD(\G_i,g)$ obtained by sublevel filtration with $g$ induces a $2D$ vector $\vec{\mu}_i=\vec{\mu}(\PD(\G_i,g))$ of size $k\times l$. Then, define EMP Persistence Image as $\M_\mu^i=\vec{\mu}_i$, where $\M_\mu^i$ is the $i^{th}$-floor of the array $\M_\mu$. Hence, $\M_\mu(\G)$ would be a $3D$-vector (array) of size $m\times k\times l$ where $m$ is the number of thresholds for the first function $f$ and $k\times l$ is the chosen resolution size for the Persistence Image $\vec{\mu}$.

\subsection{EMP for Other Types of Data} \label{sec:othertypedata}

So far, to keep the exposition simple, we described our construction in the graph setup. However, our framework is suitable for various types of data. Let $\X$ be a an image data or a point cloud. Let $f:\X\to \R$ and $g:\X\to \R$ be two filtering functions on $\X$. For example, it can be the grayscale function for image data or the density function on point cloud data.

Let $f:\X\to \R$ be the filtering function with the threshold set $\{\alpha_i\}_1^m$. Let $\X_i=f^{-1}((-\infty,\alpha_i])$. Then, we get a filtering of $\X$ as nested subspaces $\X_1\subset \X_2\subset\dots \subset \X_m=\X$. By using the second filtering function, we obtain finer filtrations for each subspace $\X_i$ where $1\leq i\leq m$. In particular, fix $1\leq i_0\leq m$ and let $\{\beta_j\}_{j=1}^n$ be the threshold set for the second filtering function $g$. Then, by restricting $g$ to $\X_{i_0}$, we get a filtering function on $X_{i_0}$, i.e., $g:\X_{i_0}\to \R$ which produces filtering $\X_{i_01}\subset \X_{i_02}\subset \dots\subset \X_{i_0n}=\X_{i_0}$. By inducing a simplicial complex $\wh{\X}_{i_0j}$ 
for each $\X_{i_0j}$, we get a filtration $\wh{\X}_{i_01}\subset \wh{\X}_{i_02}\subset \dots\subset \wh{\X}_{i_0n}=\wh{\X}_{i_0}$. This filtration results in a persistence diagram  $\PD(\X_{i_0}, g)$. For each $1\leq i\leq m$, we obtain $\PD(\X_i,g)$. Note that after getting $\{\X_i\}_{i=1}^m$ via $f$, instead of using second filtering function $g$, one can apply power filtration or Vietoris-Rips construction based on distance for each $X_{i_0}$ in order to get a different filtration $\wh{\X}_{i_01}\subset \wh{\X}_{i_02}\subset \dots\subset \wh{\X}_{i_0n}=\wh{\X}_{i_0}$. 

By using $m$ PDs, we follow a similar route to define our EMP summaries.  Let $\varphi$ be a single persistence vectorization. By applying the chosen SP vectorization $\varphi$ to each PD, we obtain a function $\varphi_i=\varphi(\PD(\X_i,g))$ on the threshold domain $[\beta_1,\beta_n]$, which can be expresses as a $1D$ (or $2D$) vector in most cases (\cref{sec:EMPexamples}). Let $\vec{\varphi}_i$ be the corresponding $1\times k$ vector for the function $\varphi_i$. Define the corresponding EMP $\M_\varphi$ as  $\M_\varphi^i=\vec{\varphi}_i$ where $\M_\varphi^i$ is the $i^{th}$ row of $\M_\varphi$. In particular, $\M_\varphi$ is
a $2D$-vector (a matrix) of size $m\times k$ where $m$ is the number of thresholds for the first filtering function $f$, and $k$ is the length of the vector $\vec{\varphi}$.

\begin{figure}[t]
    \centering
    \includegraphics[width=.99\textwidth, angle =0]{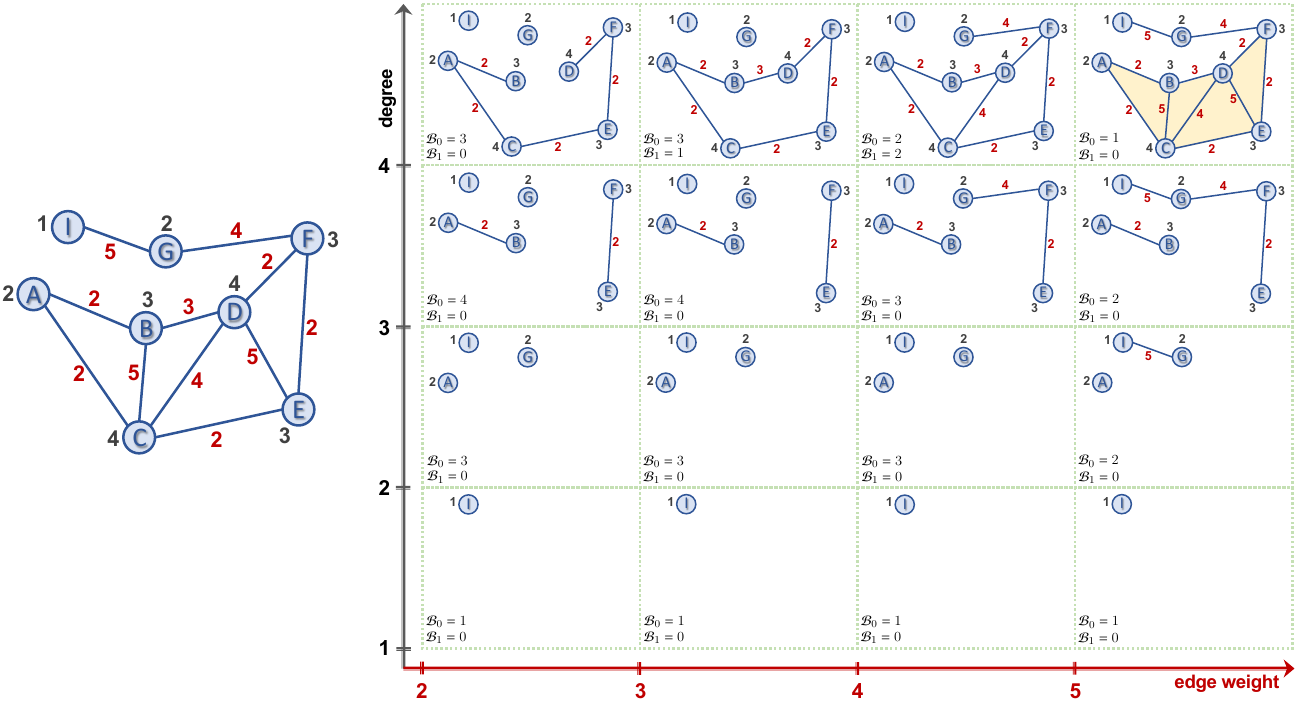}
    \caption{\scriptsize Multidimensional persistence on a graph network (original graph: left). Black numbers denote the degree values of each node whilst red numbers show the edge weights of the network. Hence, shape properties are computed on two filtering functions (i.e., degree and edge weight). While each row filters by degree, each column filters the corresponding subgraph using its edge weights. For each cell, lower left corners represent the corresponding threshold values. For each cell, $\mathcal{B}_{0}$ and $\mathcal{B}_{1}$ represent the corresponding Betti numbers.\label{Fig:ToyMP}}
\end{figure}

\subsection{EMP with Other Filtrations} \label{sec:otherfiltration}

\noindent {\bf Weight filtration} For a given weighted graph $\G=(\V,\E,\W)$, it is common to use edge weights $\W=\{\omega_{rs}\in \R^+\mid \e_{rs}\in\E\}$  to describe filtration. 
By choosing the threshold set similarly $\I=\{\alpha_i\}_1^m$ with $\alpha_1=\min\{\omega_{rs} \in \W\} <\alpha_2<\ldots<\alpha_m=\max\{\omega_{rs} \in \W\}$. For $\alpha_i\in \I$, let $\E_i=\{e_{rs}\in\V\mid \omega_{rs}\leq \alpha_i\}$. Let $\G^i$ be a subgraph of $\G$ induced by $\V_i$. 
This induces a nested sequence of subgraphs $\G_1\subset \G_2\subset \dots\subset\G_m=\G$ (See top row in \cref{Fig:ToyMP}). 

In the case of weighted graphs, one can apply the EMP framework just by replacing the first filtering (via $f$) with weight filtering. In particular, let $g:\V\to\R$ be a filtering function with threshold set $\{\beta_j\}_{j=1}^n$. Then, one can first apply weight filtering to get $\G_1\subset \dots\subset\G_m=\G$ as above, and then apply $f$ to each $\G_i$ to get a bilfiltration $\{\G_{ij}\}$ ($m\times n$ resolution). One gets $m$ PDs as $\PD(\G_i,g)$ and induce the corresponding $\M_\varphi$. Alternatively, one can change the order by applying $g$ first, and get a different filtering $\G_1\subset \G_2\subset \dots\subset\G_n=\G$ induced by $g$. Then, apply to edge weight filtration to any $\G_j$, one gets a bifiltration $\{\wh{\G}_{ji}\}$ ($n\times m$ resolution) this time. As a result, one gets $n$ PDs as $\PD(\G_i,\omega)$ and induce the corresponding $\M_\varphi$. The difference is that in the first case (first apply weights, then $g$), the filtering function plays more important role as $\M_\varphi$ uses $\PD(\G_i,g)$ while in the second case (first apply $g$, then apply weights) weights have more important role as $\M_\varphi$ is induced by $\PD(\G_j,\omega)$. Note also that there is a very different filtration method for weighted graphs by applying the following VR-complexes method.

\smallskip
\noindent {\bf Power (Vietoris-Rips) Filtration} There is a highly different filtration technique using distances between the data points in the dataset. The technique is called \textit{power filtration} for unweighted graphs~\cite{aktas2019persistence}, while it is called \textit{Vietoris-Rips filtration} for other types of data~\cite{edelsbrunner2010computational}. The idea is for a point cloud $\X=\{x_1,x_2,\dots, x_N\}$, one uses the pairwise distances $d(x_r,x_s)$ to construct the simplicial complexes in the filtration. In particular, for a threshold set $\e_1<\e_2<\dots<\e_n=diam(\X)$, one forms a Vietoris-Rips complex $\Delta_j$ by adding a $k$-simplex to $\X$ for any subset $\{x_{r_0},x_{r_1},\ldots,x_{r_k}\}$, where the pairwise distances are all $<\e_j$. If a pair of points $x_{r_1},x_{r_2}$ has distance $<\e_j$, then in the induced simplicial complex $\Delta_j$, we add an edge between the corresponding vertices $x_r$ and $x_s$. If three such points $x_{r_1},x_{r_2}, x_{r_3}$ have pairwise distances $<\e_j$, then we fill the triangle $e_{r_1r_2}\cup e_{r_2r_3}\cup e_{r_3r_1}$ with a $2$-simplex, and so on. This procedure induces in a hierarchical nested sequence of simplicial complexes $\Delta_1 \subset \Delta_2 \subset\ldots \subset \Delta_m$ that is termed \textit{Vietoris-Rips filtration} induced by the point cloud $\X$. For unweighted graphs, one takes the vertex set $\V$ as the point cloud and defines the distances $d(v_i,v_j)$ as the shortest distance in the graph where each edge has length $1$. For weighted graphs, one can do the same by defining edge lengths induced by the weights. 

One can adapt Vietoris-Rips filtrations to our EMP setting as follows. Start with a filtering function $f:\X\to \R$ with threshold set  $\{\alpha_i\}_1^m$ and obtain $\X_1\subset \X_2\subset\dots \subset \X_m=\X$ where $\X_i=f^{-1}((-\infty, \alpha_i])$. Then, apply Vietoris-Rips filtration to each $\X_{i_0}$ for threshold set $\{\e_j\}_{j=1}^n$ which produces a filtration $\wh{\X}_{i_01}\subset \wh{\X}_{i_02}\subset \dots\subset \wh{\X}_{i_0n}$ where $\wh{\X}_{i_0j}$ is the Vietoris-Rips complex of $\X_{i_0}$ for threshold $\e_j$.  
Construct  $\PD(\X_i, VR)$ of these filtrations for each $1\leq i\leq m$. The following steps are the same~\cref{sec:EMPexamples}. For a given SP vectorization $\varphi$, let $\vec{\varphi}_i$ be the corresponding $1\times k$ vector induced by $\varphi(\PD(\X_i, VR))$ with domain $[\e_1,\e_n]$. Then, define EMP $\M_\varphi$ as  $\M_\varphi^i=\vec{\varphi}_i$ where $\M_\varphi^i$ is the $i^{th}$ row of $\M_\varphi$. Again, $\M_\varphi$ is a $2D$-vector (a matrix) of size $m\times k$ where $m$ is the number of thresholds for the filtering function $f$, and $k$ is the length of the vector $\vec{\varphi}$. 

\subsection{Order of Filtration} 
\label{SecApx:AS:OrderFilt}

We need to note that for given two filtering functions $f,g$, the order is quite important for our algorithm. In particular, let $\M_\varphi(f,g)$ represent the above construction where we first apply $f$ to get filtering $\{\G_i\}_{i=1}^m$, and then we obtain $m$ different $\PD(\G_i,g)$. Hence, $\M_\varphi(f,g)$ would be a $m\times n$ matrix. On the other hand, if we apply $g$ first, we would get a filtering $\{\G^j\}_{j=1}^n$. Then, by using sublevel filtration with $g$ for each $\G^j$, we would get $n$ persistence diagrams $\PD(\G^j,f)$. Assuming we use the same thresholds for $f$ and $g$ in both orders, then $\M_\varphi(g,f)$ would give us $n\times m$ size matrices. In particular, in the first one $\M_\varphi(f,g)$ we use ``horizontal'' slicing in the bipersistence module, while in the latter $\M_\varphi(g,f)$, we use ``vertical'' slicing. 
For the question \textit{``Which function should be used first''}, the answer is that the function with more important domain information should be used as a second function ($g$ in the original construction), as we get much finer information via persistence diagram $\PD(\G_i,g)$ for the second function. 
This asymmetry enriches our method as one can combine both feature vectors obtained in different order as they do not contain the same information about the multipersistence grid. 
To observe the effect of changing the order of the filtering functions, we run experiments on three benchmark datasets, BZR\_MD, DFHR\_MD and REDDIT-BINARY. Our experiments show that in some datasets, the order can be highly important, while in others, it has a negligible effect on the performance \Cref{tab:order}. 

\begin{table}[t]
\caption{\footnotesize {\bf Filtration Order.} Impact on classification accuracy when swapping the filtration order. $S_{\V}$ and $W_{\E}$ denote sublevel filtration on nodes and weight filtration on edges, respectively.  EMP-Betti, EMP-Entropy, and EMP-Silhouette performances when using both $H_0$ and $H_1$ features. \textbf{Bold} numbers represent statistically significant superior performance with respect to other orders. \label{tab:order}}
\centering
\footnotesize
\setlength\tabcolsep{5pt}
\begin{tabular}{ l c c c c}
\toprule
{\bf Model}  & \textbf{Order} & \textbf{BZR\_MD} & \textbf{DFHR\_MD} & \textbf{Reddit-B} \\
\midrule
\textbf{EMP-B}   & $S_{\V} \leadsto W_{\E}$ & 73.20$\pm$1.57 & 74.21$\pm$0.91 & 91.03$\pm$0.22 \\
 & $W_{\E} \leadsto S_{\V}$ & \textbf{75.06$\pm$0.93} & \textbf{75.59$\pm$1.05} & 90.77$\pm$0.20 \\
\midrule
\textbf{EMP-E } & $S_{\V} \leadsto W_{\E}$ & \textbf{75.30$\pm$1.30} & \textbf{74.80$\pm$0.96} & 90.05$\pm$0.22\\
 & $W_{\E} \leadsto S_{\V}$ & 73.05$\pm$1.28 & 73.15$\pm$0.75 & 90.28$\pm$0.17  \\
\midrule
\textbf{EMP-S}  & $S_{\V} \leadsto W_{\E}$ & \textbf{77.86$\pm$0.80} & \textbf{80.13$\pm$0.94}& 88.59$\pm$0.37 \\
 & $W_{\E} \leadsto S_{\V}$ & 75.13$\pm$1.10 & 74.45$\pm$1.27 & \textbf{90.58$\pm$0.21} \\
\bottomrule
\end{tabular}
\end{table}

\subsection{Multidimensional Persistence Theory} \label{sec:MP_theory}

Multipersistence theory is under intense research because of its promise to significantly improve the performance and robustness properties of single persistence theory. While single persistence theory obtains the topological fingerprint of single filtration, a multidimensional filtration with more than one parameter should deliver a much finer summary of the data to be used with ML models. However, because of the technical issues in the theory, multipersistence has not reached to its potential yet and remains largely unexplored by the ML community. Here, we provide a short summary of these technical issues. For further details, \cite{botnan2022introduction}~gives a nice outline of current state of the theory and major obstacles. 

In single persistence, the threshold space $\{\alpha_i\}$ being a subset of $\mathbb{R}$, is totally ordered, i.e., birth time~$<$~death time for any topological feature appearing in the filtration sequence $\{\Delta_i\}$. By using this property, it was shown that “barcode decomposition” is well-defined in single persistence theory in the 1950s [Krull-Schmidt-Azumaya Theorem~\cite{botnan2022introduction}--Theorem 4.2]. This decomposition makes the persistence module $M=\{H_k(\Delta_i)\}_{i=1}^N$ uniquely decomposable into barcodes. This barcode decomposition is exactly what we call a {\em Persistence Diagram}.

However, when one goes to higher dimensions, i.e. $d=2$, then the threshold set $\{(\alpha_i,\beta_j)\}$ is no longer totally ordered, but becomes partially ordered (Poset). In other words, some indices have ordering relation $(1,2)< (4,7)$, while some do not, e.g., (2,3) vs. (1,5). Hence, if one has a multipersistence grid $\{\Delta_{ij}\}$, 
we no longer can talk about birth time or death time as there is no order any more. Furthermore, Krull-Schmidt-Azumaya Theorem is no longer true for upper dimensions~\cite{botnan2022introduction}--Section 4.2. Hence, for general multipersistence modules barcode decomposition is not possible, and the direct generalization of single persistence to multipersistence fails. On the other hand, even if the multipersistence module has a good barcode decomposition, because of partial ordering, representing these barcodes faithfully is another major problem. Multipersistence modules are an important subject in commutative algebra, where one can find the details of the topic in~\cite{eisenbud2013commutative}. 

While complete generalization is out of reach for now, several attempts have been tried to utilize MP idea~\citep{lesnick2015theory}. One of the first such novel ideas came from \cite{lesnick2015interactive} where they suggest using one-dimensional slices in the MP grid, and to get the signature of the most dominant features. Later, \cite{carriere2020multiparameter} combined several slicing directions (vineyards) and obtained a vectorization by summarizing several persistence diagrams (PDs) in these directions. Slicing techniques use the persistence diagrams of predetermined one-dimensional slices in the multipersistence grid and then combine (compress) them as one-dimensional output~\cite{botnan2022introduction}. In that respect, one major issue is that the topological summary highly depends on the predetermined slicing directions in this approach, and how to decide this direction. The other problem is the loss of information when compressing the information in various persistence diagrams.

Another novel approach to vectorizing the multipersistence module is presented by~\citet{vipond2020multiparameter}. In this study, the author effectively extends persistence landscapes into higher dimensions. This approach does not use a specific global slice direction. For every point \(\mathbf{x} \in \R^n\)—where \(n\) signifies the dimension of the multipersistence module—the \(k^{th}\)-landscape explores the widest direction in which the rank invariant has a nontrivial image. Viewed from this perspective, the Multipersistence Landscape might be regarded as a more faithful representation of the multipersistence module. While the MP Landscape can be notably effective in scenarios where vital information is derived from a few predominant topological features, such as point clouds or sparse data, its computational intensity makes it less practical for analyzing large datasets with numerous topological features. Conversely, our approach offers a less computationally intensive, yet more versatile vectorization, which can be efficiently applied to various datasets.

As explained above, the MP approach has still technical problems to reach its full potential, and there are several attempts to utilize this idea. In this paper, we do not claim to solve theoretical problems of multipersistence homology but offer a novel, highly practical multidimensional topological summary by advancing the existing methods. We use the grid directions in the multipersistence module as natural slicing directions and produce multidimensional topological summary of the data. 

As a result, these multidimensional topological fingerprints are capable of capturing very fine topological information hidden in the data. Furthermore, in the case the data provides more than two very important filtering functions, our framework easily accommodates these functions and induces corresponding substructures. Then, our EMP model captures the evolving topological patterns of these substructures and summarize them in matrices and arrays which are highly practical output formats to be used with various ML models. 

Our model is highly different from previous work mostly because of its practicality and computational efficiency. Among these, the closest method to ours is \cite{carriere2020multiparameter} which employs slicing techniques in a different way. Like us, they have predetermined slicing options (vineyards), and they compute the single persistence diagrams on these slices and combine them in a unique way by using weight functions induced by lifespans of the topological features in this collection of persistence diagrams. In our approach, we use only horizontal slices and do not compress the information. First, choosing horizontal slices is computationally very feasible to obtain persistence diagrams. Second, we offer a variety of options on how to vectorize these persistence diagrams. Hence, depending on the dataset, one can use vectorization methods that emphasize long barcodes (Silhouette with $p>1$, Entropy, Persistence Image) or the ones which consider all signals equally (Betti). Our experiments proved that these varieties of options can be quite useful as some EMP vectorizations give much better results than others in different datasets (\Cref{SecApx:AblationStudy}).

\section{Stability} \label{secapx:stab}

\subsection{Stability of Single Persistence Summaries} \label{sec:stability2}
For a given PD vectorization, stability is one of the most important properties for statistical purposes. Intuitively, the stability question is whether a small perturbation in PD cause a big change in the vectorization or not. To make this question meaningful, one needs to define what "small" and ``big" means in this context. Therefore, we need to define the distance notion, i.e., metric in the space of persistence diagrams. The most common such metric is called \textit{Wasserstein distance} (or matching distance) which is defined as follows. Let $\PD(\X^+)$ and $\PD(\X^-)$ be persistence diagrams two datasets $\X^+$ and $\X^-$ (We omit the dimensions in PDs).  Let $\PD(\X^+)=\{q_j^+\}\cup \Delta^+$ and  $\PD(\X^-)=\{q_l^-\}\cup \Delta^-$ where $\Delta^\pm$ represents the diagonal (representing trivial cycles) with infinite multiplicity. Here, $q_j^+=(b^+_j,d_j^+)\in \PD(\X^+)$ represents the birth and death times of a hole $\sigma_j$ in $\X^+$. Let $\phi:\PD(\X^+)\to \PD(\X^-)$ represent a bijection (matching). With the existence of the diagonal $\Delta^\pm$ on both sides, we make sure the existence of these bijections even if the cardinalities $|\{q_j^+\}|$ and $|\{q_l^-\}|$ are different. Then, the $p^{th}$ Wasserstein distance $\W_p$ defined as $$\W_p(\PD(\X^+),\PD(\X^-))= \min_{\phi}\biggl(\sum_j\|q_j^+-\phi(q_j^+)\|_\infty^p\biggr)^\frac{1}{p}, \quad p\in \mathbb{Z}^+.$$

Then, a vectorization (function) $\varphi(\PD(\X))$ is called \textit{stable} if $\mathrm{d}(\varphi^+,\varphi^-)\leq C\cdot \W_p(\PD(\X^+),\PD(\X^-))$ where $\varphi^\pm=\varphi(\PD(\X^\pm))$ and $\mathrm{d}(.,.)$ is a suitable metric on the space of vectorizations used. Here, the constant $C>0$ is independent of $\X^\pm$. This stability inequality interprets as the changes in the vectorizations are bound by the changes in PDs. Two nearby persistence diagrams are represented by nearby vectorizations. If a given vectorization $\varphi$ satisfies such a stability inequality for some $\mathrm{d}$ and $\W_p$, we call $\varphi$ a \textit{stable vectorization}~\cite{atienza2020stability}. Persistence Landscapes~\cite{Bubenik:2015}, Persistence Images~\cite{adams2017persistence}, Stabilized Betti Curves~\cite{johnson2021instability} and several Persistence curves~\cite{chung2019persistence} are among well-known examples of stable vectorizations.

\subsection{Stability of EMP Summaries}  \label{sec:stability}

We now show that when the source single parameter vectorization $\varphi$ is stable, then so is its induced EMP vectorization $\M_\varphi$. We give the details of the stability notion in persistence theory and examples of stable SP vectorizations in \cref{sec:stability2}.

Let $\G^+=(\V^+,\E^+)$ and $\G^-=(\V^-,\E^-)$ be two graphs. Let $\varphi$ be a stable SP vectorization with the stability equation 
\begin{equation}\label{eqn1}
\mathrm{d}(\varphi(\G^+),\varphi(\G^-))\leq C_\varphi\cdot \W_{p_\varphi}(\PD(\G^+),\PD(\G^-))
\end{equation}
for some $1\leq p_\varphi\leq \infty$. Here, $\varphi(\G^\pm)$ represent the corresponding vectorizations for $\PD(\G^\pm)$ and $\W_p$ represents Wasserstein-$p$ distance as defined in \cref{sec:stability2}. 

Now, by taking $d=2$ for EMP construction, let $f,g:\V^\pm\to \R$ be two filtering functions with threshold sets $\{\alpha_i\}_{i=1}^m$ and $\{\beta_j\}_{j=1}^n$ respectively. Then, by defining $\V^\pm_i=\{v_r\in\V^\pm\mid f(v_r)\leq \alpha_i\}$, their induced subgraphs $\{\G^\pm_i\}$ give the filtration $\wh{\G}_1\subset\wh{\G}_2\subset \dots\wh{\G}_m$ as before. For each $1\leq i\leq m$, we will have persistence diagram $\PD(\G_i,g)$ as detailed in \cref{sec:EMP}.
We define the induced matching distance between the multiple persistence diagrams as 
 {\small 
\begin{equation}\label{eqn2}
 \mathbf{D}_p(\{\PD(\G_i^+)\},\{\PD(\G_i^-)\})\\=\sum_{i=1}^m\W_p(\PD(\G^+_i, g), \PD(\G^-_i, g)).
 \end{equation}  } 
Now, we define the distance between induced EMP Summaries as 
\begin{equation}\label{eqn3}
\mathfrak{D}(\M_\varphi(\G^+),\M_\varphi(\G^-))=\sum_{i=1}^m \mathrm{d}(\varphi(\G^+_i),\varphi(\G^-_i))
\end{equation}

\begin{theorem*} \label{thm:stability}
Let $\varphi$ be a stable SP vectorization. Then, the induced EMP Vectorization $\M_\varphi$ is also stable, i.e., with the notation above, there exists $\wh{C}_\varphi>0$ such that for any pair of graphs $\G^+$ and $\G^-$, we have the following inequality.
$$\mathfrak{D}(\M_\varphi(\G^+),\M_\varphi(\G^-))\leq \wh{C}_\varphi\cdot \mathbf{D}_{p_\varphi}(\{\PD(\G^+)\},\{\PD(\G^-)\})$$
\end{theorem*}

\begin{proof} As $\varphi$ is a stable SP vectorization, for any $1\leq i\leq m$, we have $\mathrm{d}(\varphi(\G_i^+),\varphi(\G_i^-))\leq C_\varphi\cdot \W_{p_\varphi}(\PD(\G_i^+),\PD(\G_i^-))$ for some $C_\varphi>0$ by \cref{eqn1}, where 
$\W_{p_\varphi}$ is Wasserstein-$p$ distance.
Notice that the constant $C_\varphi>0$ is independent of $i$. Hence, 
\begin{align*}
\mathfrak{D}(\M_\varphi(\G^+),\M_\varphi(\G^-)) &\quad  =  & \sum_{i=1}^m \mathrm{d}(\varphi(\G^+_i),\varphi(\G^-_i)) \quad \quad \quad \quad  \quad \quad \\
\; & \quad \leq    & \sum_{i=1}^m C_\varphi\cdot \W_{p_\varphi}(\PD(\G_i^+),\PD(\G_i^-)) \quad \\
\; & \quad =  &  C_\varphi \sum_{i=1}^m \W_{p_\varphi}(\PD(\G_i^+),\PD(\G_i^-)) \quad \quad \\
\; & \quad =   & C_\varphi\cdot \mathbf{D}_{p_\varphi}(\{\PD(\G_i^+)\},\{\PD(\G_i^-)\}) \ \quad \\
\end{align*}
where the first and last equalities are due to~\cref{eqn2} and~\cref{eqn3}, while the inequality follows from~\cref{eqn1} which is true for any $i$. 
This concludes the proof of the theorem.
\end{proof}

\end{document}